\documentclass[a4paper,11pt]{article}
\usepackage[T1]{fontenc}
\usepackage{fullpage}
\usepackage{graphicx}
\usepackage{booktabs}
\usepackage{amsmath,amsthm,amssymb}
\usepackage{comment}
\usepackage{mathtools,mathrsfs,bm}
\usepackage{optidef}
\mathtoolsset{showonlyrefs}
\usepackage{times}
\usepackage{soul}
\usepackage{url}
\usepackage[hidelinks]{hyperref}
\usepackage[utf8]{inputenc}
\usepackage[small]{caption}
\usepackage{graphicx}
\usepackage{booktabs}
\usepackage{algorithm}
\usepackage{algorithmic}
\usepackage[switch]{lineno}
\usepackage[numbers,sort]{natbib}

\newtheorem{theorem}{Theorem}
\newtheorem{lemma}{Lemma}

\newtheorem{claim}{Claim}

\theoremstyle{definition}
\newtheorem{definition}{Definition}
\newtheorem{remark}[definition]{Remark}

\newcommand{\cA}{\mathcal{A}}
\newcommand{\utilityT}[1]{X_{#1}}
\newcommand{\As}{\mathcal{A}^*}
\newcommand{\Ber}[1]{\mathrm{Ber}\left( #1 \right)}
\newcommand{\Expec}[1]{\mathbb{E}\left[ #1 \right]}
\newcommand{\Eunder}[2]{\mathbb{E}_{#1}\left[ #2 \right]}

\newcommand{\Pundertwo}[2]{\mathbb{P}_{#1}\left[#2\right]}

\newcommand{\ot}{\leftarrow}
\newcommand{\OPT}{\mathrm{OPT}}
\newcommand{\ALG}{\mathrm{ALG}}
\renewcommand{\epsilon}{\varepsilon}

\DeclareMathOperator*{\argmin}{arg\,min}
\newcommand{\rad}{\mathrm{r}}

\newcommand{\Const}{C_{\rm rad}}
\newcommand{\ALGucb}{\overline{\mathrm{ALG}'}}
\newcommand{\clean}{\mathcal{E}}
\allowdisplaybreaks

\usepackage{authblk}
\title{Bandit Max-Min Fair Allocation}
\author[1]{Tsubasa Harada}
\affil{Institute of Science Tokyo}
\author[2]{Shinji Ito}
\affil{University of Tokyo}
\author[1]{Hanna Sumita}
\date{}

\begin{document}

\maketitle

\begin{abstract}
    In this paper, we study a new decision-making problem called the \textit{bandit max-min fair allocation} (BMMFA) problem. The goal of this problem is to maximize the minimum utility among agents with additive valuations by repeatedly assigning indivisible goods to them.
    One key feature of this problem is that each agent's valuation for each item can only be observed through the semi-bandit feedback, while existing work supposes that the item values are provided at the beginning of each round.
    Another key feature is that the algorithm's reward function is not additive with respect to rounds, unlike most bandit-setting problems.
    Our first contribution is to propose an algorithm that has an asymptotic regret bound of $O(m\sqrt{T}\ln T/n + m\sqrt{T \ln(mnT)})$, where $n$ is the number of agents, $m$ is the number of items, and $T$ is the time horizon. 
    This is based on a novel combination of bandit techniques and a resource allocation algorithm studied in the literature on competitive analysis.
    Our second contribution is to provide the regret lower bound of $\Omega(m\sqrt{T}/n)$.
    When $T$ is sufficiently larger than $n$, the gap between the upper and lower bounds is a logarithmic factor of $T$.
\end{abstract}

\section{Introduction}
\label{sec-intro}

In this paper, we introduce a new sequential decision-making problem, the \textit{bandit max-min fair allocation} (BMMFA) problem, in which some indivisible goods are divided among some agents in a fair manner.
The problem is motivated by a problem of designing a subscription service as follows:
the company rents items (e.g., clothes, watches, cars, etc.) to users for a certain period, collects the items when the period ends, receives feedback from users, and, based on that feedback, decides which items to rent to whom in the next period. In such a service, the company would like to make all the users as happy as possible.
How can we ensure such a fair allocation?

This problem can be regarded as an online variant of the fair allocation problem, which has been a central problem in algorithmic game theory.
The classical settings of the fair allocation problem~\cite{Bouveret2016} assume that the valuation of each agent for items is known in advance. However, this is not necessarily the case in practice.
In the above subscription service, even agents may not recognize their own valuations until they receive items. 
Therefore, this paper aims to maximize the agents' utilities while learning the valuations of agents through repeatedly allocating items.

We briefly introduce the BMMFA problem.
Let $[n] \coloneqq \{1,\dots,n\}$ be a set of $n$ agents with additive valuations, $M$ be a set of $m$ items and $T$ be the time horizon.
The value of each agent $i\in [n]$ for each item $e\in M$ follows an unknown distribution $D_{ie}$ over $[0,1]$ with the expected value $\mu_{ie}$.
For each round $t=1,\ldots, T$,
the value $v^t_{ie}$ of agent $i$ with respect to an item $e$ is sampled from $D_{ie}$ independently of the round $t$.
We denote by a matrix $a\in \{0,1\}^{n \times m}$ an allocation of items to agents, where $a_{ie}=1$ if and only if agent $i$ receives item $e$.
In each round, the algorithm decides an allocation $a^t$ of $M$ based only on the past feedback, and observes values $v^t_{ie}$ only for $(i,e)$ such that $a^t_{ie}=1$.
The utility of agent $i$ obtained at round $t$ (denoted by $X^t_i$) is the sum of the values for items which are allocated to agent $i$,
i.e., $X^t_i\coloneqq\sum_{e\in M}v^t_{ie}a^t_{ie}$. 
The utility of agent $i$ at the end of round $T$ is $\utilityT{i} \coloneqq \sum_{t=1}^T X^t_i$.

As a fairness notion, we adopt the \emph{max-min fairness}, 
which is a prominent notion in the fair allocation literature~\cite{Golovin2005,AS2010,Bouveret2016}.
Then, the sequence of allocation $a^1,\dots, a^T$ is said to be fair if the \textit{egalitarian social welfare}, which is the smallest cumulative utility among agents
$\min_{i\in [n]} \utilityT{i}$,
is maximized.

The performance of the algorithm is evaluated by an expected regret $R_T$, which is the expected difference between the egalitarian social welfare of an optimal policy and that of the algorithm.
We will detail the definition later.
We have two features in the definition of the regret compared with most other bandit problems: (a) an optimal policy knows all the expected values $\mu_{ie} = \Expec{v^t_{ie}}$ for all $(i,e)\in [n] \times M$ but may make different allocations across the $T$ rounds, and (b) an algorithm's expected reward is $\Expec{ \min_{i\in [n]} X_i}$, which is \emph{not additive} with respect to rounds.

To be more specific about (a), we assume that an optimal policy chooses a sequence of allocations $x^1, \dots, x^T$ such that $\min_{i \in [n]} \sum_{t=1}^T\sum_{e\in M} \mu_{ie} x^t_{ie}$ is maximized.
A naive definition of an optimal policy would be an allocation $\tilde{x}$ that maximizes $\min_{i \in [n]} \sum_{t=1}^T\sum_{e\in M} \mu_{ie} \tilde{x}_{ie}$.
However, this fixed-allocation policy may not be reasonable for our problem.
To see this issue, consider the case where $m<n$ and all agents have value $1$ for any items.
Any fixed-allocation policy has zero egalitarian social welfare since at least one agent receives nothing in every round, while we can achieve positive value by allocating items depending on the round.

For the point (b), $\min_i X_i$ is the fairness measure to be maximized.
The problem is that analyzing $\Expec{ \min_{i\in [n]} X_i}$ is difficult if we naively use the existing bandit techniques.
Our model is similar to the combinatorial multi-armed bandit (CMAB) problems \cite{pmlr-v28-chen13a}.
However, even in the most general setting of CMAB, the algorithm's reward is the sum of the per-round rewards ($\min_{i\in [n]} X^t_i$ in our setting), and the optimal policy selects a fixed action for all rounds. This implies that the CMAB framework does not cover our setting. 

Another similar allocation problem is studied in the context of competitive analysis~\cite{DJSW2011resourceAllocation,kawase2022online,SHPK2022}.
Roughly speaking, in each round, one item arrives and agents reveals the values for the item, and then the algorithm decides who to receive the item so that the overall egalitarian social welfare is maximized.
However, those resource allocation problems assume that the \emph{full} information $(v^t_{ie})_{i\in [n],e\in M}$
are given at the \emph{beginning} of each round,
whereas only semi-bandit feedback $(v^t_{ie})_{i,e: a^t_{ie}=1}$
is given at the end of each round in the BMMFA problem.
An optimal policy is assumed to know the realization of all the item values in advance, and the performance metrics are defined differently (see also Appendix~\ref{sec:competitive ratio} for the detail).
Therefore, the existing results do not carry to our setting.

\subsection{Our Contributions}
\label{subsec-results}

In this paper, we first define a regret that is suitable for the BMMFA problem.
Next, we propose an algorithm of this problem that achieves a regret bound of ${O}(m\sqrt{T}\ln T/n + m \sqrt{T\ln (mnT)})$ when $T$ is sufficiently large.
In addition, we provide a lower bound of $\Omega(m\sqrt{T}/n)$ on the regret.
The gap between these bounds is $O(\max\{\ln T, n\sqrt{\ln (mnT)}\})$, which is a logarithmic factor of $T$.
We highlight the techniques below.

\paragraph{Upper Bound}
Due to the features of our regret definition, it is hard to naively apply the existing approaches.
We propose an algorithm by combining techniques of regret analysis and competitive analysis.
For this, we employ a similar idea to the resource allocation algorithm proposed in \cite{DJSW2011resourceAllocation} in the context of competitive analysis.
This is similar to the multiplicative weight updated method \cite{arora2012multiplicative}.
To estimate the item values given by the semi-bandit feedback, 
we incorporate upper confidence bounds (UCB)~\cite{LAI19854} on $\mu_{ie}$ for each $(i,e)$, and adopt the error analysis used in~\cite{BKS2018} for the bandits with knapsacks problem.
Our algorithm simply allocates each item to an agent with the largest UCB, discounted by a factor depending on the past allocations.
However, the regret analysis is challenging.
If we directly analyze the regret, we need to connect the algorithm's choice (depending on UCBs) to the algorithm's reward (in terms of $v^t_{ie}$'s).
This is not easy because the reward is non-additive and UCBs do not imply future item values.
We bypass this issue by introducing a \emph{surrogate} regret, defined with expected item values.
We show that the original regret and the surrogate differ by at most $O(m\sqrt{T \ln T})$, and the surrogate regret has a bound of ${O}(m\sqrt{T}\ln T/n + m \sqrt{T\ln (mnT)} + m\ln T \ln(mnT))$.
These facts imply an upper bound on $R_T$.
We remark that our algorithm runs in $O(mn)$ time per round.

By this analysis, the average egalitarian social welfare $\Expec{\min_{i\in [n]} X_i/T}$ of our algorithm achieves per-round fairness up to an additive error of $o(1)$ when $m$ is fixed.
Here, we refer to per-round fairness\footnotemark as maximizing the expected minimum utility per round through a stochastic allocation. 
\footnotetext{Alternative choices are maximizing the minimum expected ex-ante utility $\min_{i\in [n]} \Expec{X_i}$ or using only deterministic allocations. In fact, the same guarantee holds for any choice.}

\paragraph{Lower Bound}
The proof of the lower bound primarily follows the standard method for the multi-armed bandit (MAB) problem by \cite{ACFS2002}. We first lower bound the regret by averaging over a certain class of instances for BMMFA. Then, by using Pinsker's inequality, we reduce the problem of lower bounding the regret to computing the Kullback-Leibler divergence of certain distributions. The main difference from the standard method for MAB is that we ``divide'' the problem into 
$m/n$ subproblems, each with $n$ agents and $n$ items, by treating $m/n$ as an integer.
Intuitively, the lower bound $\Omega(m\sqrt{T}/n)$ arises from the number of subproblems times the lower bound of $\Omega(\sqrt{T})$ for each subproblem. This idea of dividing the problem is similar to the proof of the lower bound for the online combinatorial optimization problem \cite{audibert2014regret}.

\subsection{Relation to Multi-Player Bandits}
\label{subsec-relation-mpb}

The situation of multiple agents choosing items has been actively studied in the context of multi-player bandits (MPB). In this problem, $n$ agents repeatedly choose one of $K$ items (or arms). In the following, we explain the difference between MPB and BMMFA from three perspectives.

The first difference is the correspondence between agents and items: in MPB, each agent chooses exactly one item per round, and there may be items that are not chosen by any agent. In the BMMFA problem, on the other hand, each item is assigned to an agent, and there may be agents who receive no items or multiple items.

The second difference is the objective function: most MPB studies aim to maximize the sum of agents' utilities and do not consider fairness among the agents.
See a survey \cite{JMLR-survey-multiplayer-bandits} for details.
However, some recent studies address the fairness issues~\cite{hossain2021fair,Jones_Nguyen_Nguyen_2023,NEURIPS2024_no_regret,bistritz2020my}. These studies aim to maximize an objective function of the form $\sum_{t=1}^T F((X_i^t)_{i\in [n]})$, where $F((X_i^t)_{i\in [n]})$ represents a fairness measure at round $t$ (e.g. \emph{Nash social welfare}~\cite{NEURIPS2024_no_regret} or the max-min fairness~\cite{bistritz2020my}).
With such objective functions, the algorithm might prioritize per-round fairness rather than overall fairness.

In fact, this approach can hinder the achievement of overall fairness because the algorithm lacks an incentive to eliminate the disparity in cumulative utility among agents\footnotemark.
On the other hand, in our setting, even if a disparity in utility occurs during the learning process, the algorithm adaptively allocates items to make an agent with small utility happier.
\footnotetext{Consider an instance with two agents and two goods $a$ and $b$. The value of $a$ is $1$ and that of $b$ is $\varepsilon \ll 1$ for both agents. Any sequence of allocations that gives one item for one agent maximizes $\sum_{t=1}^T \min_{i\in [2]} X^t_i$. However, to maximize the minimum of cumulative utilities, we need to assign $a$ to either agent once per two rounds.}

The third perspective involves the differences in the ``optimal'' policy used as a benchmark for evaluating regret. In the context of bandit problems, including prior studies addressing fairness among agents such as \cite{hossain2021fair,bistritz2020my}, the optimal policy typically consists of repeatedly making a single fixed decision. 
In contrast, BMMFA allows the optimal policy to vary its allocation in each round. In other words, we assume a stronger optimal policy compared to those in similar problems.

These distinctions make it impossible to directly compare the challenges of BMMFA with that of related problems.

\subsection{Other Related Work}
In the MAB problem, there are $K$ arms, and the algorithm chooses one arm in each round and receives a reward corresponding to the chosen arm. In recent years, there has been research into how to choose an arm that satisfies a certain constraint representing fairness. A commonly used constraint for fairness is that ``the ratio of the number of rounds each arm has been drawn to the number of rounds must be greater than a certain value''
\cite{li2019combinatorial,claure2020multi,chen2020fair,patil2021achieving}.
In the BMMFA problem, we can view an allocation as an arm.
However, as the above notion ignores the utility of agents, it is not suitable for our purpose.

There is a vast body of literature on online fair allocation in combinatorial optimization and algorithmic game theory.
Recent studies include problems with a fairness notion such as envy-freeness~\cite{Benade2018},
maximum Nash social welfare~\cite{Banerjee2022}, $p$-mean welfare~\cite{BKM2021,Cohen2024}.
They are just a few examples; see also a survey~\cite{AW2020}.
Offline sequential allocation problems have also been studied~\cite{IgarashiLNN2024,KarlJochenA2024}.
In this context, the goal is to obtain a sequence of allocations with both overall and per-round fairness guarantees.

The one-shot, offline version of the BMMFA problem has been studied in combinatorial optimization under the name of the \textit{Santa Clause problem}~\cite{Golovin2005,bezakov2005allocating,chakrabarty2009onallocating,feige2008onallocations,HaeuplerSS11}.
The problem is NP-hard even to approximate within a factor of better than 1/2 \cite{lenstra1990approximation}.
\cite{BansalS06} proposed an $\Omega(\ln \ln \ln n/ \ln \ln n)$-approximation algorithm for a restricted case. 
\cite{AS2010} provided the
first polynomial-time approximation algorithm for the general problem
and this was improved by \cite{HaeuplerSS11}.

Finally, we note that BMMFA can also be viewed as a repeated two-player zero-sum game~\cite{Cesa-Bianchi_Lugosi_2006}, where one player selects an allocation to maximize the agents' utilities, while the other player selects an ``unhappy" agent to minimize the utility.
See Appendix~\ref{sec:learning in games} for a detail.

\section{Model}
\label{sec-setup}

The bandit max-min fair allocation problem is represented by a quadruple $([n],M,T,(D_{ie})_{i\in [n],e\in M})$, where $[n]\coloneqq \{1,\ldots,n\}$ is a set of $n$ agents, $M = \{1,\dots,m\}$ is a set of $m$ items, $T$ is the time horizon, and $D_{ie}$ is a probability distribution over $[0,1]$ representing the value of agent $i$ for an item $e$.
For each $i\in [n]$ and $e\in M$, let $\mu_{ie}$ be the expected value of $D_{ie}$. Assume that $[n]$, $M$ and $T$ are known in advance, while $(D_{ie})_{i\in [n],e\in M}$ is not.

Each allocation of items to agents is expressed as an $n$-row by $m$-column $0$-$1$ matrix $a\in\{0,1\}^{n \times m}$, where
$a_{ie}=1$ if and only if agent $i$ receives item $e$ in the allocation.
Let $\cA\subseteq\{0,1\}^{n \times m}$ be a set of allocations, i.e.,
\[
    \cA = \left\{ a\in\{0,1\}^{n \times m} \, : \, \sum_{i\in [n]} a_{ie} = 1 \text{ for all } e\in M \right\}.
\]

For each round $t=1,\ldots, T$,
let $v^t_{ie}$ be a random variable drawn from $D_{ie}$.
Note that the random variables $\{v^t_{ie}:i\in[n],e\in M,t=1,\ldots,T\}$ are mutually independent and are unknown to the algorithm in this step.
In round $t$, the algorithm chooses an allocation $a^t\in\cA$ depending only on
the previous allocations $(a^s)_{s=1}^{t-1}$ and the feedback obtained up to the beginning of round $t$.
Then, the algorithm receives semi-bandit feedback: the algorithm is given the values $v^t_{ie}$ for all $(i,e)$ such that $a^t_{ie}=1$.
The reward of an algorithm, denoted by $\ALG$, is defined as the egalitarian social welfare, i.e.,
\[
   \ALG \coloneqq \min_{i\in [n]} \sum_{t=1}^T \sum_{e\in M} v^t_{ie}a^t_{ie}.
\]

The expected regret $R_T$ is defined to be the expectation of the difference between the egalitarian social welfares of an optimal policy and an algorithm. 
We assume that an optimal policy takes a sequence of allocations $x^1, \dots, x^T \in \{0,1\}^{n \times m}$ that maximizes the egalitarian social welfare with respect to the expected values, i.e., $\min_{i\in [n]} \sum_{t=1}^T \sum_{e\in M} \mu_{ie} x^t_{ie}$.
Formally, we define
\begin{align*}
    \OPT& \coloneqq \min_{i\in [n]} \sum_{t=1}^T \sum_{e\in M} v^t_{ie} x^t_{ie}, \\
    R_T &\coloneqq \Expec{\OPT-\ALG}.
\end{align*}

For the regret analysis, we introduce surrogate values of $\OPT$ and $\ALG$ as
\begin{align*}
    \OPT_\mu & \coloneqq \min_{i\in [n]} \sum_{t=1}^T \sum_{e\in M} \mu_{ie}x^t_{ie}, \\
    \ALG_\mu & \coloneqq \min_{i\in [n]} \sum_{t=1}^T \sum_{e\in M} \mu_{ie}a^t_{ie}
\end{align*}
and a \emph{surrogate} regret
\begin{align*}
    R_T^\mu \coloneqq \Expec{\OPT_\mu- \ALG_\mu}.
\end{align*}
In fact, $R_T^\mu$ is not so far from $R_T$ as the following lemma shows.
\begin{lemma}\label{lem:surrogate regret}
    $|R_T - R^\mu_T| = O(m\sqrt{T\ln T})$.
\end{lemma}
\begin{proof}
We first prove the following claim.
\begin{claim}
\label{claim-min-max}
    For $i=1,\ldots,n$, let $a_i$ and $b_i$ be any real numbers.
    Then, we have
    \begin{align}
        \min_{i} a_i-\min_{i} b_i \leq \max_{i} (a_i-b_i) \ \text{ and } \ 
        |\min_{i} a_i-\min_{i} b_i| \leq \max_{i} |a_i-b_i|. 
    \end{align}
\end{claim}
\begin{proof}
Let $i^*_a\in\argmin_{i}a_i$ and $i^*_b\in\argmin_{i}b_i$.
Then,
\begin{align}
    \min_{i} a_i-\min_{i} b_i
    = \min_{i} a_i-b_{i^*_b} \le a_{i^*_b}-b_{i^*_b}
    \le \max_{i}(a_i-b_i) \le \max_{i}|a_i-b_i|.
\end{align}
This implies that the first inequality holds. Similarly to the above, we also have
\begin{align}
    \min_{i} b_i-\min_{i} a_i
    \leq b_{i^*_a}-a_{i^*_a} \leq \max_{i} (b_i-a_i)
    \le \max_{i} |b_i-a_i|.
\end{align}
Then, the second inequality holds.
\end{proof}

Let $a^1,\dots, a^T$ be the random variables representing the algorithm's choice.
For each agent $i$, let $U_i$ and $P_i$ be the cumulative utilities of agent $i$ with respect to the realization of item values and the expected values, respectively, i.e., 
$$U_i = \sum_{t=1}^T \sum_{e\in M} v^t_{ie} a^t_{ie}, \text{ and } P_i = \sum_{t=1}^T \sum_{e\in M} \mu_{ie} a^t_{ie}.$$
We observe that $\lambda_i \coloneqq \Expec{U_i} = \Expec{\sum_{t=1}^T \sum_{e\in M} \mu_{ie} a^t_{ie}} = \Expec{P_i}$ because $v^t_{ie}$ is independent of $a^t$.
We remark that $U_i$ and $P_i$ are the sum of $T$ random variables over $[0, m]$.

By the Hoeffding's inequality, for any $\delta > 0$, we have
\begin{align*}
    \Pr[|U_i - \lambda_i| \geq \delta] \leq 2 \exp\left(-\frac{2\delta^2}{m^2 T}\right), \text{ and }
    \Pr[|P_i - \lambda_i|\geq \delta] \leq 2 \exp\left(-\frac{2\delta^2}{m^2 T}\right).
\end{align*}
Therefore, by the union bound, it holds that
\begin{align*}
    \Pr[\exists i\in [n], |P_i - U_i| \geq 2\delta]
    & \leq \Pr[\exists i, \ |U_i - \lambda_i| \geq \delta \text{ or } |P_i - \lambda_i| \geq \delta]\\
    &\leq \sum_{i\in [n]} (\Pr[|U_i - \lambda_i| \geq \delta]+\Pr[|P_i - \lambda_i| \geq \delta])
    \leq 4n \exp\left(-\frac{2\delta^2}{m^2 T}\right).
\end{align*}
Let $\clean_{\delta}$ be the event such that $|P_i - U_i| < 2\delta$ for any $i\in [n]$.
Then, we have
\begin{align}
    \left|\Expec{\min_{i\in [n]} U_i} - \Expec{\min_{i\in [n]} P_i} \right|
    &\leq \Expec{\Bigl|\min_{i\in [n]} U_i-\min_{i\in [n]} P_i\Bigr|} \\
    &\leq \Expec{\Bigl|\min_{i\in [n]} U_i-\min_{i\in [n]} P_i\Bigr|\middle| \clean_{\delta}}\Pr[\clean_{\delta}]  +mT/n \cdot 4n \exp\left(-\frac{2\delta^2}{m^2 T}\right) \\
    &\leq \Expec{\max_{i\in[n]}|P_i-U_i|\middle| \clean_{\delta}} +4mT \exp\left(-\frac{2\delta^2}{m^2 T}\right) \\
    &\leq 2\delta + 4mT \exp\left(-\frac{2\delta^2}{m^2 T}\right),
    \label{eq:lemma1-delta}
\end{align}
where the third inequality is due to Claim~\ref{claim-min-max}.
If we set $\delta = m\sqrt{T\ln T}$, then the RHS in \eqref{eq:lemma1-delta} is $2m\sqrt{T\ln T}+4m/T$.
Therefore, we obtain
\begin{align}
\label{eq:regret-alg}
\left|\Expec{\ALG} - \Expec{\ALG_{\mu}} \right| = O(m\sqrt{T\ln T}).
\end{align}

Let $x^1,\dots, x^T$ be the choice of an optimal policy with respect to the expected item values, i.e., they achieve $$\max_{x^1,\dots,x^T \in \cA} \min_{i\in [n]} \sum_{t=1}^T \sum_{e\in M} \mu_{ie} x^t_{ie}.$$
We remark that $x^1,\dots, x^T$ are chosen deterministically.
Then, in a similar way to the above discussion, we have
\begin{align}
    \left|\Expec{\min_{i\in [n]} \sum_{t=1}^T \sum_{e\in M} v^t_{ie} x^t_{ie}} - \min_{i\in [n]} \sum_{t=1}^T \sum_{e\in M} \mu_{ie} x^t_{ie} \right|
    = O(m\sqrt{T\ln T}). \label{eq:regret-opt}
\end{align}

By combining~\eqref{eq:regret-alg} and \eqref{eq:regret-opt}, it follows that
\begin{align*}
    |\Expec{\OPT-\ALG} - \Expec{\OPT_{\mu}-\ALG_{\mu}}|
    &= |\Expec{\OPT-\OPT_{\mu}} - \Expec{\ALG-\ALG_{\mu}}|\\
    &= O(m\sqrt{T\ln T}).
\end{align*}
This completes the proof.
\end{proof}

Furthermore, $\OPT_\mu$ is upper bounded by the optimal value of the following LP:
\begin{align}
  \begin{array}{rll}
    \max_{P,x}        & T \cdot P  &   \\
    \text{s.t.} & P\le \sum_{e \in M} \mu_{ie} x_{ie} & (\forall i\in [n]),                    \\
                & \sum_{i\in [n]}x_{ie}=1                            & (\forall e\in M),                  \\
                & 0 \leq x_{ie} \leq 1                              & (\forall i\in [n],\ \forall e \in M).
  \end{array}
  \tag{$\mathrm{LP}$}\label{eq:LPE}
\end{align}
Indeed, if we set $\hat{x}_{ie} = \sum_{t=1}^T x^t_{ie} /T$ ($i\in [n], e\in M$), then $\hat{x}$ is a feasible solution to~\eqref{eq:LPE}.
Let $(P^*,x^*)$ be an optimal solution of~\eqref{eq:LPE}.
We will evaluate $T\cdot P^* - \Expec{\ALG_\mu}$ to obtain an upper bound on $R^{\mu}_T$.

Note that \eqref{eq:LPE} can be interpreted as maximizing the minimum expected per-round utility when a stochastic allocation is allowed.
Since $P^*$ upper bounds the maximum ``expected minimum'' per-round utility, bounding $T\cdot P^* - \Expec{\ALG_\mu}$ leads to per-round fairness on average; see also Remark~\ref{remark:per-round}.

In what follows, we assume $P^*>0$ because otherwise $R^\mu_T=0$.
Moreover, intuitively, if $P^*$ is sufficiently small, then a per-round utility $\sum_{e\in M} \mu_{ie}a^t_{ie}$ of any agent $i$ is not far less than $P^*$, and hence $\ALG_\mu$ is also close to $P^*T$.
Therefore, the difficulty of our problem lies in the case when $P^*$ is large.
This is a nature of max-min fair allocation problems.
Indeed, existing results in~\cite{DJSW2011resourceAllocation,kawase2022online} for competitive analysis also assume that the offline optimal value is sufficiently large.

\section{Algorithm}
\label{sec-algorithm}

In this section, we describe an algorithm with a regret bound of ${O}(m\sqrt{T}\ln T/n + m \sqrt{T\ln (mnT)}+m\ln T \ln(mnT))$.
The regret bound will be shown in the next section.
The algorithm is based on the resource allocation algorithm in \cite{DJSW2011resourceAllocation,kawase2022online}.
The brief description of (a multiple-item variant of) the algorithm is as follows.
It is assumed that the values $v^t_{ie}$ for all $(i,e)$ are given at the beginning of each round.
Let $\epsilon>0$ be a parameter, which will be set later.
We denote by $u_i^t$ the cumulative utility of agent $i$ at the end of round $t$.
In each round $t$, the algorithm chooses an allocation $a^t$ that maximizes a total sum of utilities with respect to item values discounted with $u_i$.
More specifically, $a^t$ achieves $\max_{a\in \cA} \sum_{i\in [n], e\in M} (1-\epsilon)^{u^{t-1}_i/m} v^t_{ie} a_{ie}$.

Due to the feedback model, a direct application of the above resource allocation algorithm is impossible in our setting.
It is also not clear whether the existing result carries to due to the different definition of $\OPT$.

To address those issues, we estimate each value using an \emph{upper confidence bound} (UCB), and reconstruct the performance evaluation by incorporating the error analysis used in \cite{BKS2018}.

For $v\in \mathbb{R}_+$ and $N \in \mathbb{Z}_+$, let
\begin{align*}
    \rad(v, N)= \sqrt{{\Const \cdot v}/{N}}+{\Const}/{N},
\end{align*}
where $\Const$ is a positive constant independent of $v$ and $N$.
For each round $t$ and $(i,e) \in [n] \times M$, we define 
\begin{align}
    \bar{v}^{t}_{ie} = \hat{v}_{ie} + \rad(\hat{v}_{ie}, N_{ie,t})
    \label{eq:UCB}
\end{align}
as a UCB of $v^{t}_{ie}$, where $N_{ie,t}$ is the number of rounds in which item $e$ is assigned to agent $i$ in the first $t-1$ rounds and $\hat{v}_{ie}$ is the average of the $N_{ie,t}$ samples of $v^t_{ie}$.
For this setting of UCBs, the following useful result is known.
\begin{theorem}[\cite{BKS2018}]\label{thm:clean}
    Let $\hat{\nu}$ be the average of $N$ independent samples from a distribution over $[0,1]$ with expectation $\nu$.
    For each $\Const>0$, it holds that 
    \begin{align}
        \Pr [|\nu-\hat{\nu}| \leq \rad(\hat{\nu},N) \leq 3\rad(\nu, N)] \geq 1-e^{-\mathrm{\Omega}(\Const)}.
    \end{align}
    This holds even if $X_1,\dots, X_N \in [0,1]$ are random variables, $\hat{\nu} = \frac{1}{N}\sum_{t=1}^N X_t$ is the sample average, and $\nu = \frac{1}{N}\sum_{t=1}^N \mathbb{E}[X_t \mid X_1,\dots,X_{t-1} ]$.
\end{theorem}
We will set the constant $\Const=\Theta(\ln (mnT))$.
Then, by using the union bound, we have
\begin{equation}
    \mu_{ie} \in [\hat{v}^t_{ie}-\rad(\hat{v}^t_{ie},N_{ie,t}), \hat{v}^t_{ie}+\rad(\hat{v}^t_{ie},N_{ie,t})]
    \label{eq:ucb-success}
\end{equation}
for any $(i,e) \in [n] \times M$ and round $t$
with probability at least $1-\frac{1}{T}$.
We call this event a \emph{clean execution}~\cite{BKS2018} and denote it by $\clean$.

Our algorithm is summarized in Algorithm~\ref{alg:proposed-ucb}.
We devote the first $n$ rounds to collect one sample of each item value.
At the subsequent rounds $t$, assuming $\bar{v}^t_{ie}$ as an estimation of $\mu_{ie}$, we choose an allocation $a^t$ maximizing $\sum_{i\in [n],e \in M}  (1-\epsilon)^{u^{t-1}_i/m} \bar{v}_{ie}^{t} \cdot a^t_{ie}$.
We can obtain $a^t$ easily just by allocating each item $e$ to the agent with the largest discounted UCB for $e$.

\begin{algorithm}[tb]
\caption{Allocation algorithm}
\label{alg:proposed-ucb}
\textbf{Parameter}: $\epsilon\in(0,1)$.
\begin{algorithmic}[1] 
\FOR{\(t=1,\dots,n\)}
    \STATE Assign all items to agent $t$ and receive values $v^{t}_{te}$ for each $e \in M$.
\ENDFOR
\STATE Set $\bar{v}^{n+1}_{ie}$ as in~\eqref{eq:UCB} for each $i\in [n]$ and $e\in M$.
\STATE Let $u^n_i=0$ for each $i\in [n]$.
\FOR{\(t=n+1,\dots, T\)}
    \STATE Let $a^{t}$ be an allocation $a \in \cA$ maximizing
    \[\sum_{i\in [n]} \sum_{e \in M} (1-\epsilon)^{\frac{1}{m}u^{t-1}_i} \bar{v}_{ie}^{t} \cdot a_{ie}.\]
    \label{line:choice}
    \STATE Receive values $v^{t}_{ie}$ for each $(i,e)$ such that $a^t_{ie}=1$.
    \STATE Set $u^t_{i}\ot u^{t-1}_i + \sum_{e\in M} \bar{v}_{ie}^{t}a^t_{ie} $ for each $i\in [n]$.
    \STATE  Set $\bar{v}^{t+1}_{ie}$ accordingly as in~\eqref{eq:UCB}.
\ENDFOR
\end{algorithmic}
\end{algorithm}

\section{Regret Analysis}
\label{sec-regret-analysis}

The main goal of this section is to prove the following theorem, which provides an asymptotic guarantee on $R_T$.

\begin{theorem}\label{thm:alg-main}
The regret $R_T$ for Algorithm \ref{alg:proposed-ucb} is bounded as
\[
    R_T 
    \leq m + P^*
    + W'\left(
        \epsilon
        + n e^{-\frac{\epsilon^2 W'}{2m}}
    \right)
    + O(m\sqrt{T\ln T} + err)
    ,
\]
where $W'=P^*(T-n)$, $err=O(\sqrt{\Const m^2 T}+\Const m\ln T)$ and $\Const=\Theta(\ln(mnT))$.
If $T \geq e^{\frac{2m}{P^*}}+n$, by setting
$
    \epsilon = \ln(T-n)/\sqrt{T-n}
    ,
$
we have
$$
    R_T=
    O\left(
        \frac{m}{n}\sqrt{T}\ln T + err
    \right).
$$
\end{theorem}
Note that weakening the assumption on $T$ to $T\geq n$ yields a regret bound of $O(m\sqrt{T\ln n/n}+err)$.
We show this in Section~\ref{sec:another regret bound}.

To prove Theorem~\ref{thm:alg-main}, it suffices to show that $R^{\mu}_T\leq m + P^*
    + W'\left(
        \epsilon
        + n e^{-\frac{\epsilon^2 W'}{2m}}
    \right)+O(err)$ by Lemma~\ref{lem:surrogate regret}.

As described before, $R^{\mu}_T \leq TP^*-\Expec{\ALG_\mu}$.
For each agent $i$, let $X^t_{i}$ be a random variable representing the reward of the agent at round $t$ with respect to the expected item values, i.e., $X^t_{i} = \sum_{e\in M} \mu_{ie}a^t_{ie}$.
For notational convenience, let $W'=P^*(T-n)$ and let $\ALG'_\mu=\min_{i\in [n]}\sum_{t=n+1}^T X^t_i$.
The following simple calculations allow us to ignore regret in the first $n$ rounds:
\begin{align}
    R^{\mu}_T
    & \le  P^*T-\Expec{\min_{i\in [n]}\sum_{t=n+1}^T X^t_i}
    =P^*n +W'-\Expec{\ALG'_\mu}
    \leq m+W'-\Expec{\ALG'_\mu}.
    \label{eq:nrounds}
\end{align}
Then, in the rest of this section,
we bound $W'-\Expec{\ALG'_\mu}$.

For each agent $i$, let $\bar{X}^t_{i}$ be a random variable representing the reward of the agent at round $t$ if values are replaced with their UCBs, i.e., $\bar{X}^t_{i} = \sum_{e\in M} \bar{v}^{t}_{ie}a^t_{ie}$.
In addition, let $\ALGucb\coloneqq\min_{i\in [n]}\sum_{t=n+1}^T \bar{X}^t_i$ be the total reward of Algorithm~\ref{alg:proposed-ucb} after round $n$ with respect to the UCBs.
We first claim that $\Expec{\ALGucb}$ is not far from $\Expec{\ALG'_\mu}$ in the follwing lemma.
Let $err \coloneqq \Expec{\ALGucb}-\Expec{\ALG'_\mu}$.
\begin{lemma}\label{lem:diff UCB}
It holds that $err =O(\sqrt{\Const m^2 T} + \Const m\ln T)$.
\end{lemma}
\begin{proof}
By Claim~\ref{claim-min-max} in the proof of Lemma~\ref{lem:surrogate regret}, it follows that
\begin{align}
   \ALGucb-\ALG'_\mu
    \leq \max_{i\in [n]} \sum_{t=n+1}^T \left(\bar{X}^t_{i} - {X}^t_{i} \right).
\end{align}
Thus, in what follows, we show that  
\begin{equation}
\label{eq:lemma-main}
    \sum_{t=n+1}^T \left(\bar{X}^t_{i} - {X}^t_{i} \right)
    = O(\sqrt{\Const m^2 T} + \Const m\ln T)
\end{equation} 
for every $i\in [n]$ when the clean execution $\clean$ occurs.
Once this is shown, it follows from $\Pr[\clean]\geq 1-1/T$ that
\begin{align}
    \Expec{\ALGucb}-\Expec{\ALG'_\mu}
    &\leq \begin{multlined}[t]
        O(\sqrt{\Const m^2 T} + \Const m\ln T) 
        + 1/T\cdot mT
    \end{multlined} \\
    &= O(\sqrt{\Const m^2 T} + \Const m\ln T),
\end{align}
and the proof is complete.
To prove \eqref{eq:lemma-main},
we use the following claim.

\begin{claim}[\cite{BKS2018}]
\label{claim-r-leq-3r}
    For any $\nu,\hat{\nu}\in [0,1]$,
    if $|\nu-\hat{\nu}|\le \rad(\hat{\nu},N)$, then $\rad(\hat{\nu}, N) \leq 3\rad(\nu, N)$.
\end{claim}

When we assume the clean execution, we have
\begin{align}
    \sum_{t=n+1}^T \left(\bar{X}^t_{i} - {X}^t_{i} \right) 
    &=\sum_{t=n+1}^T \sum_{e\in M} (\bar{v}^t_{ie}-\mu_{ie})a^t_{ie} \\
    &\le 2\sum_{t=n+1}^T \sum_{e\in M}\rad(\hat{v}^t_{ie}, N_{ie,t})a^t_{ie} \\
    &\le 6\sum_{t=n+1}^T \sum_{e\in M}\rad(\mu_{ie}, N_{ie,t})a^t_{ie} \\
    &\le 6\sum_{t=1}^T \sum_{e\in M}\rad(\mu_{ie}, N_{ie,t})a^t_{ie} \\
    &= 6\sum_{e\in M} \sum_{l=1}^{N_{ie,T}} \rad(\mu_{ie}, l) \\
    &= 6\sum_{e\in M}
        O( \sqrt{\Const\mu_{ie}N_{ie,T}} + \Const\ln N_{ie,T}) \\
    &\leq O\Bigl(\sqrt{\Const m \sum_{e\in M} \mu_{ie}N_{ie,T}} + \Const m\ln N_{ie,T}\Bigr) \\
    &\leq O\Bigl(\sqrt{\Const m^2 T } + \Const m\ln T\Bigr),
\end{align}
where the first inequality is due to \eqref{eq:ucb-success}, the second inequality is due to Claim \ref{claim-r-leq-3r}, and the forth inequality is the Cauchy-Schwartz inequality.
\end{proof}

Lemma~\ref{lem:diff UCB} implies that we only need to evaluate
$W'-\Expec{\ALGucb}$.
We proceed based on the idea in~\cite{DJSW2011resourceAllocation, kawase2022online}.

By the union bound and Markov's inequality, the probability that $\ALGucb$ is at most $(1-\varepsilon)W'$ is 
\begin{align}
  \Pr\left[\min_{i\in [n]}\sum_{t=n+1}^T \bar{X}^t_{i}\le (1-\epsilon)W'\right]
   & \le \sum_{i\in [n]}\Pr\left[\sum_{t=n+1}^T \bar{X}^t_{i}\le (1-\epsilon)W'\right]  \\
   & = \sum_{i\in [n]}\Pr\left[(1-\epsilon)^{\frac{1}{m}\sum_{t=n+1}^T \bar{X}^t_{i}}\ge (1-\epsilon)^{(1-\epsilon)W'/m}\right]                              \\
   & \le \sum_{i\in [n]}\mathbb{E}\left[(1-\epsilon)^{\frac{1}{m}\sum_{t=n+1}^T \bar{X}^t_{i}}\right]/(1-\epsilon)^{(1-\epsilon)W'/m}. \label{eq:iid_bound1}
\end{align}
If the rightmost value in~\eqref{eq:iid_bound1} is sufficiently small, then we can bound the regret by $m +{O}(\varepsilon W')$ with high probability.
For $s=n,n+1,\dots,T$, let us define $\Phi(s)$ as
\begin{align}
  \Phi(s)\coloneqq \sum_{i\in [n]} (1-\epsilon)^{\frac{1}{m}\sum_{t=n+1}^s \bar{X}^t_{i}} \cdot \Bigl(1-\frac{\epsilon P^*}{m}\Bigr)^{T-s}.
\end{align}
We note that the rightmost value in \eqref{eq:iid_bound1} is equal to $\mathbb{E}[\Phi(T)]/(1-\epsilon)^{(1-\epsilon)W'/m}$.

\begin{lemma}\label{lem:potential_monotone}
    In a clean execution of Algorithm~\ref{alg:proposed-ucb}, $\Phi(s)$ is monotone non-increasing in $s$.
\end{lemma}
\begin{proof}
  Since the feasible region of \eqref{eq:LPE} is a subset of the convex hull of integral allocations of $M$ to $[n]$, we can decompose $x^*$ as a convex combination of some integral allocations $y^1,\dots, y^k$ so that $x^* = \sum_{j \in [k]} \lambda_j y^j$, where $\lambda_j\geq 0$ ($\forall j\in [k]$) and $\sum_{j\in [k]} \lambda_j =1$.
  We note that $y^1, \dots, y^k$ are not necessarily optimal solutions to~\eqref{eq:LPE}.
  Then, for $s=n+1,\dots,T-1$, letting $\alpha_i = (1-\epsilon)^{\frac{1}{m} \sum_{t=n+1}^s \bar{X}^t_{i}}$, we can see that
  \begin{align}
    \Phi(s+1)
     & = \sum_{i\in [n]}\alpha_i\cdot (1-\epsilon)^{\frac{1}{m} \bar{X}^{s+1}_{i}}\cdot \bigl(1-\frac{\epsilon}{m}P^*\bigr)^{T-s-1}\\
     & \le \sum_{i\in [n]}\alpha_i\cdot \bigl(1-\frac{\epsilon}{m} \bar{X}^{s+1}_{i} \bigr)\cdot \bigl(1-\frac{\epsilon}{m}P^*\bigr)^{T-s-1}\\
     & \le \sum_{i\in [n]}\alpha_i\cdot\bigl(1- \frac{\epsilon}{m}P^*\bigr)\cdot \bigl(1-\frac{\epsilon }{m}P^*\bigr)^{T-s-1} \\
     & = \sum_{i\in [n]}\alpha_i\cdot\bigl(1-\frac{\epsilon}{m}P^*\bigr)^{T-s}
    = \Phi(s).
  \end{align}
  Here, the first inequality holds by $\bar{X}^{s+1}_{i}/m\in [0,1]$ and $(1-\epsilon)^x\le 1-\epsilon x$ for any $x\in[0,1]$. 
  As for the second inequality,
  \begin{align}
  \sum_{i\in [n]} \alpha_i \cdot \bar{X}^{s+1}_{i}
  &\ge 
  \sum_{j=1}^k \lambda_j \sum_{i\in [n]} \alpha_i \sum_{e\in M}\bar{v}^{s+1}_{ie}y^j_{ie}
  = \sum_{i\in [n]} \alpha_i \sum_{e\in M}\bar{v}^{s+1}_{ie}x^*_{ie}
  \end{align}
  holds for each $i\in [n]$ by the choice of $a^t$ in line~\ref{line:choice}.
  Since we assume a clean execution, it further holds that
  $
  \sum_{e\in M}\bar{v}^{s+1}_{ie}x^*_{ie} \geq \sum_{e\in M} \mu_{ie} x^*_{ie} \geq P^*$.
\end{proof}
The proof of Lemma~\ref{lem:potential_monotone} requires a connection between a utility with respect to the UCBs and an optimal policy.
This task is made easier if we use the surrogate regret.

\begin{lemma}\label{lem:potential_zero}
  It holds that $\Phi(n)/(1-\epsilon)^{(1-\epsilon)W'/m}\le n\cdot e^{-\frac{\epsilon^2W'}{2m}}$.
\end{lemma}
\begin{proof}
  By definition of $\Phi$ and $1-x\le e^{-x}$ for any $x$, we have
  \begin{align*}
    \frac{\Phi(n)}{(1-\epsilon)^{\frac{(1-\epsilon)W'}{m}}}
    =\frac{\sum_{i\in [n]}\bigl(1-\frac{\epsilon P^*}{m}\bigr)^{T-n}}{(1-\epsilon)^{(1-\epsilon)\frac{W'}{m}}}
    \le \frac{n\cdot e^{-\epsilon \frac{W'}{m}}}{(1-\epsilon)^{(1-\epsilon)\frac{W'}{m}}}.
  \end{align*}
  This is bounded by $n\cdot e^{-\frac{\epsilon^2 W'}{2m}}$ since $\frac{1}{(1-\epsilon)^{(1-\epsilon)}}\le e^{\epsilon-\epsilon^2/2}$ for any $\epsilon \in[0,1)$.
\end{proof}

Now we are ready to prove Theorem~\ref{thm:alg-main}.
By applying Lemmas~\ref{lem:potential_monotone} and~\ref{lem:potential_zero} to \eqref{eq:iid_bound1}, we see that 
\begin{align}
    \Pr\Bigl[\ALGucb \le (1-\epsilon)W'\Bigr]
    &\le \Pr\Bigl[\ALGucb \le (1-\epsilon)W' \mid \clean\Bigr] + \frac{1}{T}\\
    &\le \frac{\mathbb{E}[\Phi(T)\mid \clean]}{(1-\epsilon)^{(1-\epsilon)W'/m}} + \frac{1}{T}\\
    &\le \frac{\Phi(n)}{(1-\epsilon)^{(1-\epsilon)W'/m}} + \frac{1}{T}
    \le n\cdot e^{-\frac{\epsilon^2 W'}{2m}} + \frac{1}{T}.
\end{align}
This implies that
\begin{align}
    \label{eq:w-ealg}
      W'-\Expec{\ALGucb}
      &\leq \epsilon W' + (n\cdot e^{-\frac{\epsilon^2 W'}{2m}} + {1}/{T})W'
      \leq \epsilon W' + n\cdot e^{-\frac{\epsilon^2 W'}{2m}}W' + P^*.
\end{align}
This together with \eqref{eq:nrounds} and Lemmas~\ref{lem:surrogate regret} and~\ref{lem:diff UCB} implies that
\begin{align}
    R_T
    &\leq R^\mu_T + O(m\sqrt{T\ln T})\\
    &\leq m+W' - \Expec{\ALGucb} + O(m\sqrt{T\ln T}+err)\\
    &\leq m + \epsilon W' + n\cdot e^{-\frac{\epsilon^2 W'}{2m}}W' + P^* + {O}(m\sqrt{T\ln T}+\sqrt{\Const m^2 T} + \Const m\ln T).
    \label{eq:regret}
\end{align}

Let $T'\coloneqq T-n~(\geq e^{\frac{2m}{P^*}})$, and 
we set $\epsilon = \frac{\ln T'}{\sqrt{T'}}$.
Then it follows that
\begin{align}
    \epsilon W' + n\cdot e^{-\frac{\epsilon^2 W'}{2m}}W'
    &= P^*\sqrt{T'} \ln T'+nP^*T'^{1-\frac{P^*}{2m}\ln T'}
    \le P^*\sqrt{T'} \ln T'+nP^*.
\end{align}
Therefore, from Equation~\eqref{eq:regret}, we finally see that
\begin{align}
    R_T &\le
    m + P^*\sqrt{T'} \ln T'+nP^* + P^* + {O}(\sqrt{\Const m^2 T} + \Const m\ln T)  \\
    &= \begin{multlined}[t]
        O\Bigl(\frac{m}{n}\sqrt{T} \ln T 
        + \sqrt{\Const m^2 T} + \Const m\ln T \Bigr).
    \end{multlined}
\end{align}
This completes the proof of Theorem~\ref{thm:alg-main}.

\subsection{Another Regret Bound}\label{sec:another regret bound}
In this subsection, we show that Algorithm~\ref{alg:proposed-ucb} has a regret bound of $O(m\sqrt{T\ln n/n}+err)$ even under a weaker assumption of $T\geq n$.

\begin{theorem}\label{thm:another}
    If $T\geq n$, the regret $R_T$ for Algorithm \ref{alg:proposed-ucb} with $\varepsilon = \sqrt{n \ln n/T}$ is bounded as
\[
    R_T 
    = O\left(m\sqrt{\frac{T\ln n}{n}} + m\sqrt{T \ln T} + err\right),
\]
where $err=O(\sqrt{\Const m^2 T}+\Const m\ln T)$ and $\Const=\Theta(\ln(mnT))$.
\end{theorem}
\begin{proof}
Let $i^*\in\argmin_{i}\sum_{t=n+1}^T\bar{X}_{i}^{t}$.
Under the clean execution $\mathcal{E}$,
we have
\begin{align}
	(1 - \epsilon)^{
		\frac{1}{m} \sum_{t=n+1}^T \bar{X}_{i^*}^{t}}
	\le
	\sum_{i \in [n]}
	(1 - \epsilon)^{\frac{1}{m} \sum_{t=n+1}^T \bar{X}_i^{t}}
	=
	\Phi(T)
	\le
	\Phi(n)
	=
	n
	\left(
		1 - \frac{\epsilon P^*}{m}
	\right)^{T-n},
\end{align}
where the inequality follows from Lemma~\ref{lem:potential_monotone}.
By taking the logarithm of both sides,
we have
\begin{align}
	\frac{1}{m} \sum_{t=n+1}^T \bar{X}_{i^*}^{t}
	\ln (1-\epsilon)
	\le
	(T-n) 
	\ln
	\left(
		1 - \frac{\epsilon P^*}{m}
	\right)
	+
	\ln n.
\end{align}
As we have
$
- x -x^2 \le \ln(1-x)
\le -x
$
for $x \le 1/2$,
we obtain
\begin{align}
	\left(
		-
		\epsilon
		-
		\epsilon^2
	\right)
	\frac{1}{m} 
	\sum_{t=n+1}^T \bar{X}_{i^*}^{t}
	&\le
	-
	(T-n)
	\frac{\epsilon P^*}{m}
	+
	\ln n 
	=
	-
	\frac{\epsilon W'}{m}
	+
	\ln n.
\end{align}
Therefore, under the clean execution $\clean$, it follows that
\begin{align}\label{eq:diff-realization}
	W'
	-
	\overline{\mathrm{ALG}'}
	&=
	W'
	-
	\sum_{t=n+1}^T \bar{X}_{i^*}^{t} \\
	&\le
	\epsilon
	\sum_{t=n+1}^T \bar{X}_{i^*}^{t}
	+
	\frac{m}{\epsilon} \ln n \\
	&\le
	\frac{m}{n} \epsilon T
	+
	\frac{m}{\epsilon} \ln n
    =2m\sqrt{\frac{T\ln n}{n}},
\end{align}
where the last equality holds if we set $\epsilon=\sqrt{n\ln n/T}$.
Since $\clean$ occurs with probability at least $1-1/T$, we have
\begin{align}
    W'-\Expec{\ALGucb}
    &\leq 2m\sqrt{\frac{T\ln n}{n}} + P^*(T-n)/T 
    = O\left(m\sqrt{\frac{T\ln n}{n}}\right). \label{eq-another-regret}
\end{align}
By applying~\eqref{eq-another-regret} to the proof of Theorem~\ref{thm:alg-main} instead of~\eqref{eq:w-ealg}, 
we obtain $R_T=O(m\sqrt{T \ln n / n}+m\sqrt{T\ln T}+err)$.
\end{proof}

\begin{remark}\label{remark:per-round}
    We observe the outcome of Algorithm~\ref{alg:proposed-ucb} almost achieves per-round fairness on average across rounds.
    Here we mean per-round fairness by attaining the maximum expected minimum utility per round with a stochastic allocation, whose value is bounded by $P^*$.
    Indeed, \eqref{eq:regret-alg} and~\eqref{eq-another-regret} imply that, we have 
    \begin{align}
        P^* - \Expec{\min_{i\in [n]} \frac{1}{T}\sum_{t=1}^T X^t_i} 
        &\leq P^*-\frac{1}{T}\Expec{\ALG'_\mu}+ \frac{1}{T}O(m\sqrt{T\ln T})\\
        &\leq \frac{1}{T}\left(W'-\Expec{\ALGucb} +nP^* +err + O(m\sqrt{T\ln T}) \right)
    \end{align}
    and this is $o(1)$ when $m$ and $n$ are fixed.
\end{remark}

\subsection{Extension to Matroid Constraints}\label{sec:matroid}
In this subsection, we explain that almost the same regret bounds in Theorems~\ref{thm:alg-main} and~\ref{thm:another} hold even for a setting where an allocation in each round must satisfy a matroid constraint for agents.

For a finite ground set $U$ and a family $\mathcal{I} \subseteq 2^U$ of subsets, the pair $(U, \mathcal{I})$ is called a \emph{matroid} if the following properties are satisfied: (i) $\emptyset \in \mathcal{I}$, (ii) $X\subseteq Y\in\mathcal{I}$ implies $X\in\mathcal{I}$, and (iii) if $X,Y\in\mathcal{I}$ and $|X|<|Y|$ then there exists $e\in Y\setminus X$ such that $X\cup\{e\}\in\mathcal{I}$.
For a matroid $(U, \mathcal{I})$, a member of $\mathcal{I}$ is called an \emph{independent set}, and a function $r\colon 2^U\to\mathbb{Z}_+$ such that $r(X)=\max\{|X'|\mid X'\subseteq X,\,X'\in\mathcal{I}\}$ for any $X\subseteq U$ is called the \emph{rank} function.
See e.g., a textbook~\cite{schrijver2003combinatorial} for the basic results on matroids.

Let $\mathcal{M}$ be a matroid over $[n] \times M$, and let $r$ be the rank function of $\mathcal{M}$.
We suppose that we can identify the independence of an item set by an oracle in a constant time.
Each agent's bundle in each round must be an independent set of $\mathcal{M}$.
Then the set $\cA$ of feasible allocations is 
$$\cA=\Bigl\{a \in \{0,1\}^{n \times m} \, : \, \sum_{i\in [n]} a_{ie}\leq 1~(\forall e\in M), \ \sum_{(i,e)\in Y} a_{ie} \leq r(Y)~(\forall Y \subseteq [n] \times M)\Bigr\}.$$
Examples of this setting include situations that each agent can receive at most $K$ items in one round, and that items are categorized and each agent can receive at most $K_j$ items from the $j$-th category.

For this setting, we modify Algorithm~\ref{alg:proposed-ucb} as follows.
At the beginning, the algorithm spends at most $mn$ rounds to obtain one sample for each $(i,e)$, which can be done by allocating one item to one agent in each round.
In line~\ref{line:choice}, the algorithm chooses an allocation $a^t$ that maximizes $\sum_{(i,e)\in [n]\times M} (1-\epsilon)^{\frac{1}{m}u^{t-1}_i} \bar{v}_{ie}^{t} \cdot a_{ie}$ subject to $a \in \cA$.
This can be done efficiently by using a classic algorithm to solve the maximum-weight common independent set problem.

In analysis, we modify \eqref{eq:LPE} to replace the second constraint with the constraints in $\cA$. 
Since $\cA$ is the set of common independent sets of two matroids (namely, a partition matroid and $\mathcal{M}$), the feasible region of~\eqref{eq:LPE} without the first constraint is still an integral polytope.
Therefore, by adapting the analysis, we observe that $R_T = O(m^2 + m\sqrt{T}\ln T /n + err)$ when $T \geq mn+e^{\frac{2m}{P^*}}$, and $R_T=O(m^2+m\sqrt{T \ln n / n}+m\sqrt{T\ln T}+err)$ when $T\geq mn$.

\section{Lower Bound}

In this section, we prove a lower bound on $R^\mu_T$ and $R_T$.
\begin{theorem}
\label{thm-lowerbound}
    For bandit max-min fair allocation problem with $m\geq n$, 
    the surrogate regret $R^{\mu}_T$ of any algorithm is at least $\Omega(m\sqrt{T}/n)$.

    If $T\geq \max\{n, m^2\} \geq 2$ and $m/n \geq \lceil 2338 \ln T \rceil$ in addition, then the regret $R_T$ is also at least $\Omega(m\sqrt{T}/n)$.
\end{theorem}

\subsection{Lower Bound for \texorpdfstring{$R^{\mu}_T$}{R{mu}T}}

We first prove the first part of Theorem~\ref{thm-lowerbound} (a lower bound on $R^{\mu}_T$) for any deterministic algorithm based on the idea of \cite{ACFS2002,audibert2014regret}, and then extend the proof to any randomized algorithm.
In the following, we use $\ALG$ to denote both an algorithm and its reward.

Fix any deterministic algorithm.
Let $b$ be a positive integer and $m=nb$. 
We use two index
$(j,k)$ ($j=1,\ldots,n$ and $k=1,\ldots,b$)
to represent one item $e$.
An item $(j,k)$ is called the $j$-th item in the $k$-th item block $k$.
Then, we can write the set of allocations as follows:
\[
    \cA= \left\{
        a \in \{0, 1\}^{n \times n \times b}:
            \sum_i a_{i,j,k}=1 \text{ for } \forall j,k
    \right\}.
\]
Similarly, we define the set of optimal allocations as follows:
\begin{align}
    \As= \left\{
        a \in \{0, 1\}^{n \times n \times b} \, : \,
        \begin{array}{l}
        \sum_i a_{i,j,k}=1 \text{ for } \forall j,k\\
            \sum_j a_{i,j,k}=1 \text{ for } \forall i,k
        \end{array}
    \right\}.
\end{align}
For any $\alpha\in\As$, $j\in [n]$ and $k\in [b]$,
let $I_{\alpha,j,k}$ be the unique $i$ such that $\alpha_{i,j,k}=1$.

Now we design a hard instance for the problem.
Let $\epsilon \in (0,1)$ be a parameter. 
We first choose $\alpha\in\As$ arbitrarily and
set a distribution $D_{i,j,k}$ (of agent $i$ for item $(j,k)$) to be a Bernoulli distribution $\Ber{1/2 + \epsilon \alpha_{i,j,k}}$.
We refer to $(\Ber{1/2 + \epsilon \alpha_{i,j,k}})_{i,j,k}$ as $\alpha$-adversary.
Moreover, for each $\alpha\in\As$ and $k'\in[b]$,
we also define another adversary called $(\alpha-k')$-adversary as follows:
$D_{i,j,k}=\Ber{1/2}$ if $k=k'$, and
$D_{i,j,k}=\Ber{1/2 + \epsilon \alpha_{i,j,k}}$ otherwise.
Note that for an allocation $\beta \in \cA^*$, the $(\alpha-k')$-adversary is the same as the $(\beta-k')$-adversary
if
$\alpha_{i,j,k}=\beta_{i,j,k}$
for each
$i\in[n]$, $j\in[n]$ and $k\in[b]\setminus\{k'\}$.
When we use an $\alpha$-adversary,
for each agent $i$ and each item $(j,k)$, we say that 
$(i,j,k)$ is a \emph{correct assignment} if $\alpha_{i,j,k}=1$.
We use
$\Pundertwo{\alpha}{\cdot}$
and
$\Eunder{\alpha}{\cdot}$
to denote the conditional probability and expectation
when we choose an $\alpha$-adversary at first.

Let $\alpha^*\in \cA^*$ be the most unfavorable adversary that minimize the reward.
Let $\mu_{i,j,k}$ be the expected value of each $D_{i,j,k}$.
We denote $N_{\alpha,k}\coloneqq \sum_{t,i,j}\alpha_{i,j,k} a^t_{i,j,k}$. Then we have
\begin{align}
    \Eunder{\alpha^*}{\ALG_{\mu}}
    &= \Eunder{\alpha^*}{\min_i\sum_{t,j,k} \mu_{i,j,k} a^t_{i,j,k} } \\
    &\leq \frac{1}{|\As|} 
        \sum_{\alpha \in \As}
        \Eunder{\alpha}{ \min_i\sum_{t,j,k} \mu_{i,j,k} a^t_{i,j,k}} \\
    &\leq \frac{1}{n|\As|} 
        \sum_{\alpha \in \As}
        \Eunder{\alpha}{ \sum_{t,i,j,k} \mu_{i,j,k} a^t_{i,j,k} } \\
    &= \frac{1}{n|\As|}
        \sum_{\alpha \in \As}
        \Eunder{\alpha}{ \sum_{t,i,j,k} \left( \frac12 + \epsilon\alpha_{i,j,k} \right) a^t_{i,j,k} } \\
    &= \frac12 bT + \frac{\epsilon}{n|\As|}
        \sum_{\alpha\in\As} \sum_{k=1}^b 
        \Eunder{\alpha}{\sum_{t,i,j} \alpha_{i,j,k} a^t_{i,j,k}} \\
    &= \frac12 bT + \frac{\epsilon}{n|\As|}
        \sum_{k=1}^b \sum_{\alpha\in\As} 
        \Eunder{\alpha}{N_{\alpha,k}}.
        \label{eq:ALG-ub}
\end{align}

Next, we show the following lemma.
\begin{lemma}
\label{lem-A1}
For each $0 < \epsilon \leq 1/4$,
\[
    \Eunder{\alpha}{N_{\alpha,k}}
    \leq \Eunder{\alpha-k}{N_{\alpha,k}}
        + 2\epsilon nT \sqrt{ \Eunder{\alpha-k}{N_{\alpha,k}} }.
\]
\end{lemma}
\begin{proof}
Let $\sigma^t\in\{0,1\}^{n\times b}$ denote the feedback that
the algorithm observes at round $t$, i.e.,
the $(j,k)$ entry of $\sigma^t$ is $v^t_{i',j,k}$
where $i'$ is the agent who receives $(j,k)$.
In addition, for $t=1,\ldots,T$,
we denote by 
$S_t = (\sigma^1,\ldots,\sigma^t) \in \{0,1\}^{n\times b\times t}$
all feedback observed up to round $t$.
In the rest of the proof,
we use the following notation on the KL-divergence:
\begin{align*}
    K_t &\coloneqq \sum_{S_t \in \{0,1\}^{n\times b\times t}}
        \Pundertwo{\alpha-k}{S_t}
        \ln \frac{\Pundertwo{\alpha-k}{S_t}}{\Pundertwo{\alpha}{S_t}}, \\
    K'_t &\coloneqq \sum_{S_t \in \{0,1\}^{n\times b\times t}}
        \Pundertwo{\alpha-k}{S_t}
        \ln \frac{\Pundertwo{\alpha-k}{\sigma^t|S_{t-1}}}{\Pundertwo{\alpha}{\sigma^t|S_{t-1}}}.
\end{align*}
By the chain rule, we have $K_T=\sum_{t=1}^T K'_t$.
Since the algorithm is assumed to be deterministic,
we can treat $N_{\alpha,k}$ as a function $f$ of $S_T$.
Then, the following holds:
\begin{align}
    \Eunder{\alpha}{N_{\alpha,k}}-\Eunder{\alpha-k}{N_{\alpha,k}}
    &= \Eunder{\alpha}{f(S_T)}-\Eunder{\alpha-k}{f(S_T)} \\
    &= \sum_{S_T} 
        f(S_T)(\Pundertwo{\alpha}{S_T}
        -\Pundertwo{\alpha-k}{S_T}) \\
    &\leq \sum_{S_T:\Pundertwo{\alpha}{S_T}>\Pundertwo{\alpha-k}{S_T}}
        f(S_T)(\Pundertwo{\alpha}{S_T}
        -\Pundertwo{\alpha-k}{S_T}) \\
    &\leq nT \sum_{S_T:\Pundertwo{\alpha}{S_T}>\Pundertwo{\alpha-k}{S_T}}
        (\Pundertwo{\alpha}{S_T}
        -\Pundertwo{\alpha-k}{S_T}) \\
    &= \frac{nT}{2} \sum_{S_T}
        |\Pundertwo{\alpha}{S_T}
        -\Pundertwo{\alpha-k}{S_T}| \\
    &\leq \frac{nT}{2} \sqrt{ 2K_T } = \frac{nT}{2} \sqrt{ 2\sum_{t=1}^T K'_t },
    \label{eq-EN-EN-leq-nT2-sqrt-k't}
\end{align}
where $N_{\alpha,k}\coloneqq f(S_T)$,
the first inequality is due to $N_{\alpha,k}\leq nT$
and the last inequality is due to the Pinsker's inequality.
$K'_t$ is computed as follows.
\begin{claim}
\label{claim-k't}
    Fix any $S_{t-1}$ and let $P(S_{t-1})$ be the number of correct assignments $(i,j,k')$ in $a^t$ such that $k'=k$, i.e.,
    $
        P(S_{t-1})\coloneqq\sum_{i,j=1}^n \alpha_{i,j,k}a^t_{i,j,k}.
    $
    Then, we have
    \begin{equation}
    \label{eq:k't}
         K'_t=\frac12\ln\frac1{1-4\epsilon^2} \Eunder{\alpha-k}{P(S_{t-1})}.
    \end{equation}
\end{claim}
If this claim holds, by substituting the above \eqref{eq:k't} into \eqref{eq-EN-EN-leq-nT2-sqrt-k't}, it follows that
\begin{align*}
    \Eunder{\alpha}{N_{\alpha,k}}-\Eunder{\alpha-k}{N_{\alpha,k}} 
    &\leq \frac{nT}{2} \sqrt{ 2\sum_{t=1}^T K'_t } \\
    &\leq \frac{nT}{2} \sqrt{
        \ln\frac1{1-4\epsilon^2} \sum_{t=1}^T
            \Eunder{\alpha-k}{P(S_{t-1})}
    } \\
    &=\frac{nT}{2} \sqrt{
        \ln\frac1{1-4\epsilon^2} \Eunder{\alpha-k}{N_{\alpha,k}}
    } \\
    &\leq 2\epsilon nT \sqrt{\Eunder{\alpha-k}{N_{\alpha,k}}},
\end{align*}
where the last inequality follows from the convexity of $-\ln(1-x)$ and $0<\epsilon\leq 1/4$.
Hence, in the remainder, we show Claim~\ref{claim-k't}.
\begin{proof}[Proof of Claim~\ref{claim-k't}]
To begin with, we fix any $\sigma^t\in\{0,1\}^{n\times b}$ and find $\Pundertwo{\alpha}{\sigma^t|S_{t-1}}$ and $\Pundertwo{\alpha-k}{\sigma^t|S_{t-1}}$.
For $h=0,1$, define the following parameters:
\begin{itemize}
    \item $k_h(\sigma^t)$: The number of correct assignments $(i,j,k')$ in $a^t$
    such that $k'=k$ and the algorithm observes feedback $h$, i.e., the $(j,k')$-entry of $\sigma^t$ is $h$,
    \item $l_h(\sigma^t)$: The number of correct assignments $(i,j,k')$ in $a^t$
    such that $k'\neq k$ and the algorithm observes feedback $h$, i.e., the $(j,k')$-entry of $\sigma^t$ is $h$.
\end{itemize}
Note that $P(S_{t-1})=k_0(\sigma^t)+k_1(\sigma^t)$.
When $\sigma^t$ is clear from the context,
we abbreviate $k_h(\sigma^t)$ and $l_h(\sigma^t)$ as $k_h$ and $l_h$ respectively.
Then, we get
\begin{align*}
    \Pundertwo{\alpha-k}{\sigma^t|S_{t-1}}
    &= \left(\frac12\right)^{nb-(l_0+l_1)}\left(\frac12-\epsilon\right)^{l_0}\left(\frac12+\epsilon\right)^{l_1}, \\
    \Pundertwo{\alpha}{\sigma^t|S_{t-1}}
    &= \left(\frac12\right)^{nb-(k_0+k_1+l_0+l_1)}\left(\frac12-\epsilon\right)^{k_0+l_0}\left(\frac12+\epsilon\right)^{k_1+l_1},
\end{align*}
and
\begin{align*}
\MoveEqLeft
    \Pundertwo{\alpha-k}{\sigma^t|S_{t-1}} 
        \ln \frac{\Pundertwo{\alpha-k}{\sigma^t|S_{t-1}}}
            {\Pundertwo{\alpha}{\sigma^t|S_{t-1}}}\\
    &= \left(\frac12\right)^{nb-(l_0+l_1)}
        \left(\frac12-\epsilon\right)^{l_0}
        \left(\frac12+\epsilon\right)^{l_1}
        \ln \frac{(1/2)^{k_0+k_1}}{(1/2-\epsilon)^{k_0}(1/2+\epsilon)^{k_1}}.
\end{align*}

\noindent
Next, we compute the value
\begin{align*}
    value
    \coloneqq \sum_{\sigma^t\in\{0,1\}^{n\times b}}
        \Pundertwo{\alpha-k}{\sigma^t|S_{t-1}} 
            \ln \frac{\Pundertwo{\alpha-k}{\sigma^t|S_{t-1}}}
                {\Pundertwo{\alpha}{\sigma^t|S_{t-1}}}
\end{align*}
for each fixed $S_{t-1}$.
Since $\ALG$ is deterministic, 
$P(S_{t-1})\coloneqq k_0+k_1$ and $Q(S_{t-1})\coloneqq l_0+l_1$ are determined only by $S_{t-1}$.
$P(S_{t-1})$ and $Q(S_{t-1})$ are abbreviated by $P$ and $Q$ respectively
when $S_{t-1}$ is fixed.
The number of $\sigma' \in \{0,1\}^{n\times b}$
such that $k_0(\sigma')=k_0$ and $l_0(\sigma')=l_0$
is $2^{nb-(P+Q)}\binom{P}{k_0}\binom{Q}{l_0}$.
Therefore, 
\begin{align*}
    value
    &= \left(\frac12 \right)^P \sum_{k_0=0}^P
        \sum_{l_0=0}^Q
            \binom{P}{k_0}\binom{Q}{l_0} 
                \left(\frac12-\epsilon\right)^{l_0}
                \left(\frac12+\epsilon\right)^{Q-l_0}
                \ln \frac{(1/2)^{P}}{(1/2-\epsilon)^{k_0}(1/2+\epsilon)^{P-k_0}}\\
    &= \left(\frac12 \right)^P
        \sum_{l_0=0}^Q
            \binom{Q}{l_0}
            \left(\frac12-\epsilon\right)^{l_0}
            \left(\frac12+\epsilon\right)^{Q-l_0} 
            \sum_{k_0=0}^P
            \binom{P}{k_0}
            \ln \frac{(1/2)^{P}}{(1/2-\epsilon)^{k_0}(1/2+\epsilon)^{P-k_0}} \\
    &= \left(\frac12 \right)^P
        \sum_{k_0=0}^P
            \binom{P}{k_0}
            \ln \frac{(1/2)^{P}}{(1/2-\epsilon)^{k_0}(1/2+\epsilon)^{P-k_0}} \\
    &\eqqcolon \left(\frac12 \right)^P \sum_{k_0=0}^P \phi(P,k_0).
\end{align*}
In addition, we have
\begin{align*}
    \phi(P,k_0) + \phi(P,P-k_0)
    &= \binom{P}{k_0} \left( \ln
        \frac{(1/2)^{P}}{(1/2-\epsilon)^{k_0}(1/2+\epsilon)^{P-k_0}} + \ln
        \frac{(1/2)^{P}}{(1/2-\epsilon)^{P-k_0}(1/2+\epsilon)^{k_0}}
    \right) \\
    &= P \binom{P}{k_0} \ln\frac1{1-4\epsilon^2}.
\end{align*}
This implies that
\begin{align*}
    value
    &= \left(\frac12 \right)^P \sum_{k_0=0}^P \phi(P,k_0)
    = \left(\frac12 \right)^P \frac12 \sum_{k_0=0}^P 
       P \binom{P}{k_0} \ln\frac1{1-4\epsilon^2}
    = \frac{P}{2} \ln\frac1{1-4\epsilon^2}.
\end{align*}
Finally, we determine the exact value of $K'_t$ as 
\begin{align*}
    K'_t 
    &= \sum_{S_t \in \{0,1\}^{n\times b\times t}}
        \Pundertwo{\alpha-k}{S_t}
        \ln \frac{\Pundertwo{\alpha-k}{\sigma^t|S_{t-1}}}
            {\Pundertwo{\alpha}{\sigma^t|S_{t-1}}} \\
    &= \sum_{S_{t-1}} \Pundertwo{\alpha-k}{S_{t-1}}
        \cdot value
        \\
    &= \sum_{S_{t-1}}
        \Pundertwo{\alpha-k}{S_{t-1}} \frac{P(S_{t-1})}{2} \ln\frac1{1-4\epsilon^2} \\
    &= \frac12\ln\frac1{1-4\epsilon^2} \Eunder{\alpha-k}{P(S_{t-1})}.
\end{align*}
\end{proof}
This completes the proof of Lemma~\ref{lem-A1}.
\end{proof}

By the definition of an $(\alpha-k)$-adversary, we obtain
\begin{align*}
    \sum_{\alpha\in\As} \Eunder{\alpha-k}{N_{\alpha,k}}
    &= \frac{1}{n!}
        \sum_{\beta\in\As}
        \sum_{\alpha: (\alpha-k)=(\beta-k)}
            \Eunder{\alpha-k}{N_{\alpha,k}} \\
    &= \frac{1}{n!}
        \sum_{\beta\in\As}
            \Eunder{\beta-k}{
                \sum_{\alpha: (\alpha-k)=(\beta-k)}
                    \sum_{t,i,j}
                        a^t_{i,j,k} \alpha_{i,j,k}
            } \\
    &= \frac{1}{n!}
        \sum_{\beta\in\As}
            \Eunder{\beta-k}{
                \sum_{t,i,j}
                    a^t_{i,j,k}
                    \sum_{\alpha: (\alpha-k)=(\beta-k)}
                        \alpha_{i,j,k}
            } \\
    &= \frac{1}{n!}
        \sum_{\beta\in\As} \left(
            nT\cdot (n-1)!
        \right)
    =|\As|T.
\end{align*}
By this result, Lemma \ref{lem-A1} and the Cauchy-Schwartz inequality,
it follows that
\begin{align*}
    \sum_{\alpha}
        \Eunder{\alpha}{N_{\alpha,k}}
    &\leq |\As|T + 2\epsilon nT
        \sqrt{ |\As|\cdot|\As|T }
\end{align*}
and then, 
\begin{align}
    \label{eq-epsilon-n-A-sum-sum-EN-alpha-k}
    \frac{\epsilon}{n|\As|}
        \sum_{k=1}^b \sum_{\alpha\in\As} 
        \Eunder{\alpha}{N_{\alpha,k}}
    \leq \epsilon bT \left(
            \frac{1}{n} + 2\epsilon \sqrt{T}
        \right).
\end{align}
Note that $\OPT_\mu=(1/2+\epsilon)bT$.
By tuning $\epsilon = 1/(8\sqrt{T})$, we finally have the following lower bound on $R^{\mu}_T$ for the deterministic algorithm:
\begin{align}
    \OPT_\mu-\Eunder{\alpha^*}{\ALG_{\mu}}
    &\geq \epsilon bT \left(
            1 - \frac{1}{n} - 2\epsilon\sqrt{T}
        \right) \\
    &\geq \epsilon bT \left( \frac12 - 2\epsilon\sqrt{T} \right) 
    = \frac{1}{32}b\sqrt{T}=\Theta\left(\frac{m}{n}\sqrt{T}\right).
    \label{eq:lower bound}
\end{align}

Let $\ALG=\{\ALG_{x}\}_{x\in\mathcal{X}}$ be any randomized algorithm
which choose a deterministic algorithm $\ALG_x$ with probability $p(x)$.
Let $\alpha^*$ be the most unfavorable adversary.
Similarly to the case for deterministic algorithms, we have
\begin{align*}
    \Eunder{\alpha^*}{\ALG_{\mu}}
    &\leq \frac12 bT + \frac{\epsilon}{n|\As|}
        \sum_{k=1}^b \sum_{\alpha\in\As} 
        \Eunder{\alpha}{N_{\alpha,k}} \\
    &= \frac12 bT + \Eunder{x}{ \frac{\epsilon}{n|\As|}
        \sum_{k=1}^b \sum_{\alpha\in\As} 
        \Eunder{\alpha,x}{N_{\alpha,k}}
        },   
\end{align*}
where $\Eunder{x}{\cdot}$ is the expectation
with respect to the randomness of $\ALG$ and
$\Eunder{\alpha,x}{\cdot}$ is the expectation
with respect to an $\alpha$-adversary for a deterministic algorithm $\ALG_x$.
In the section above, we showed~\eqref{eq-epsilon-n-A-sum-sum-EN-alpha-k}
for all deterministic algorithms.
This implies that
\[
    \Eunder{x}{ \frac{\epsilon}{n|\As|}
        \sum_{k=1}^b \sum_{\alpha\in\As} 
        \Eunder{\alpha}{N_{\alpha,k}}}
    \leq \epsilon bT \left(
            \frac{1}{n} + 2\epsilon \sqrt{T}
        \right)
\]
and the lower bound $\Omega(m\sqrt{T}/n)$ also holds for randomized algorithms.
Then we see the first part of Theorem~\ref{thm-lowerbound}.

\subsection{Lower Bound for \texorpdfstring{$R_T$}{RT}}

Next we show a lower bound on $R_T$.
Under the assumptions $T\geq \max\{n, m^2\}$ and $m/n \geq \lceil 2338 \ln T \rceil$, we claim that $\OPT_\mu$ is close enough to $\Expec{\OPT}$.
\begin{lemma}\label{claim:opt}
    When $T\geq \max\{n, m^2\}$ and $m/n \geq \lceil 2338 \ln T \rceil$, it holds that $\OPT_\mu \leq \Expec{\OPT} + (\frac{1}{32}-\frac{1}{1000}) b\sqrt{T}$.
\end{lemma}
\begin{proof}
By the construction of the hard instance, an optimal policy allocates each item $(j,k)$ to agent $I_{\alpha, j, k}$ in every round.
Each agent $i$ receives exactly one item from each block.
Therefore, $\OPT_\mu=(1/2+\epsilon)bT$.
To analyze $\Expec{\OPT}$, let $X^t_{i,k} \in \{0,1\}$ be a random variable that represents a reward for agent $i$ from the item in the $k$-th block at round $t$.
Then $X^t_{i,k} \sim \Ber{1/2 + \epsilon}$, where $\epsilon= 1/(8\sqrt{T})$.

For each $i\in [n]$, let $U_i = \sum_{t=1}^T \sum_{k=1}^b X^t_{i,k}$.
It holds that $\OPT = \min_{i\in [n]} U_i$, and $\Expec{U_i} = (1/2+\epsilon)bT = \OPT_\mu$ for every $i\in [n]$.
By the Chernoff bound, we have
\[
    \Pr[U_i \leq (1-\delta) \OPT_\mu ] \leq \exp(-\delta^2 \OPT_\mu /2)
\]
for any $\delta \in [0, 1)$.
Then, by using the union bound, we can see that
\[
    \Pr[\exists i\in [n], U_i \leq (1-\delta) \OPT_\mu] \leq n \cdot \exp(-\delta^2 \OPT_\mu /2).
\]

We set $\delta = \frac{c}{\sqrt{(1/2+\epsilon)T}}$, where $c=\frac{1}{32}-\frac{1}{500}$.
By the assumption, $b=\frac{m}{n} \geq \lceil 2338 \ln T \rceil \geq \lceil \frac{2}{c^2} \ln T \rceil$.
We observe that $be^{-\frac{c^2b}{2}} < \frac{m}{n} \frac{1}{T}$ since $b \geq \frac{2}{c^2} \ln T$, and $\frac{m}{n} \frac{1}{T} \leq \frac{1}{n\sqrt{T}}$ since $T\geq \max\{n, m^2\} \geq 2/c^2$.

For this setting, $\OPT_\mu - \min_{i\in [n]} U_i \leq cb\sqrt{(1/2+\epsilon)T}$ holds with probability at least $1-n e^{-\frac{c^2b}{2}}$.
A trivial upper bound is $\OPT_\mu - \min_{i\in [n]} U_i \leq \OPT_\mu$.
Combining all things, we see that
\begin{align*}
    \OPT_\mu - \Expec{\OPT}
    &\leq cb\sqrt{(1/2+\epsilon)T}+(1/2+\epsilon)bT n e^{-\frac{c^2b}{2}}\\
    & \leq cb\sqrt{T}+(\frac{T}{2}+\frac{\sqrt{T}}{8})b n e^{-\frac{c^2b}{2}}\\
    &\leq (c+\frac{1}{2b}+\frac{1}{8bT}) b\sqrt{T}\\
    &\leq (\frac{1}{32}-\frac{1}{1000}) b\sqrt{T},
\end{align*}
where the last inequality holds since $\frac{1}{8b}(4+\frac{1}{\sqrt{T}}) < \frac{17}{4 \cdot 8b} \leq \frac{1}{1000}$.
\end{proof}
Since we can show that
$
    \Eunder{\alpha^*}{\ALG}
    \le \frac12 bT + \frac{\epsilon}{n|\As|}
        \sum_{k=1}^b \sum_{\alpha\in\As} 
        \Eunder{\alpha}{N_{\alpha,k}}
$
in a way similar to~\eqref{eq:ALG-ub},
the lower bound $b\sqrt{T}/32$ established in the first part of Theorem~\ref{thm-lowerbound} is also a lower bound of $\OPT_{\mu}-\Expec{\ALG}$.
This holds also for randomized algorithms.
Plugging Lemma~\ref{claim:opt} into \eqref{eq:lower bound}, we finally see that
\begin{align*}
    \Expec{\OPT} - \Eunder{\alpha^*}{\ALG}
    \geq \frac{1}{1000}b\sqrt{T}
    = \Theta\left(\frac{m}{n}\sqrt{T}\right).
\end{align*}
Then the second part of Theorem~\ref{thm-lowerbound} follows.

\section{Conclusion and Discussion}
In this paper, we introduced the bandit max-min fair allocation problem. 
We have proposed an algorithm with a regret bound of $O(m\sqrt{T}\ln T/n + m\sqrt{T \ln(mnT)})$ when $T$ is sufficiently large, and showed a lower bound $\Omega(m\sqrt{T}/n)$ on the regret.
Thus, when $T$ is sufficiently large, the bounds matches up to a logarithmic factor of $T$.

We remark that the regret bounds also apply to the setting where the goal is maximizing the minimum ``expected'' utility $\ALG_{\rm E} = \min_{i} \Expec{\sum_{t,e} \mu_{ie} a^t_{ie}}$ and the regret is defined as $\OPT_{\mu}-\ALG_{\rm E}$.
In this setting, we have $\ALG_{\rm E} \geq \ALG_{\mu}$ and then similar proofs work to derive the same bounds.

One future work is to close the gap between the upper and lower bounds on $R_T$.
Weakening the assumption on $T$ in Theorem~\ref{thm:alg-main} and analyzing a lower bound directly using $\Expec{\OPT}$ are also open.
Another potential future work is to extend the problem setting to reflect practical situations.
For example, in a subscription service, the rental period can be different depending on situations.
It would be possible to improve a regret if users let us know what they probably dislike (i.e., item $e$ with $\mu_{ie}$ being almost zero).
We believe that such an extension of the problem provides insight into real-world applications.

\section*{Acknowledgments}
This work was partially supported by the joint project of Kyoto University and Toyota Motor Corporation, titled ``Advanced Mathematical Science for Mobility Society'', 
JST ERATO Grant Number JPMJER2301, %HS
JST ASPIRE Grant Number JPMJAP2302, %HS 
and 
JSPS KAKENHI Grant Numbers 
JP21K17708, %HS
JP21H03397, %HS
and JP25K00137. %HS

\bibliographystyle{plain}
\bibliography{ref}

\appendix

\section{Competitive Ratio}\label{sec:competitive ratio}
We detail the performance metric for the online resource allocation in the competitive analysis.
In the literature, the optimal policy (optimal offline algorithm) is usually assumed to know all the realizations $\{v^t_{ie}\}_{i,e,t}$ in advance, and 
$\OPT$ and $\ALG$ are defined as
\begin{align*}
    \OPT &=\Expec{\max_{x^1,\dots,x^T} \min_{i\in [n]} \sum_{t=1}^T v^t_{ie}x^t_{ie}} \\
    \ALG &=\Expec{\min_{i\in [n]} \sum_{t=1}^T v^t_{ie}a^t_{ie}}.
\end{align*}
A \emph{competitive} ratio of an algorithm is the ratio $\ALG/\OPT$.
This definition of $\OPT$ is suitable for the case when $\{v^t_{ie}\}_{i,e}$ are given to the algorithm at the beginning of each round $t$.

\section{Relationship to Learning in Games}\label{sec:learning in games}

Our problem setting and algorithm can be interpreted within the framework of learning in games.
It is a known result that in a repeated two-player zero-sum game, if both players uses no-regret policies (in terms of their own rewards) to decide a (mixed) strategy in each round, then the pair of the time-average strategies converges to a Nash equilibrium.
See, e.g., \cite[Chapter 7]{Cesa-Bianchi_Lugosi_2006} for details.

In our problem, the optimal value $P^*$ of~\eqref{eq:LPE} can be seen as the game value of a two-player zero-sum game in which one player (max) chooses an allocation and the other player (min) chooses an agent.
A pure strategy of the min player corresponds to a unit vector in $\{0,1\}^n$, and the set of mixed strategies is the convex hull of the unit vectors, which we denote by $\Delta(n) \subseteq [0,1]^n$.
For the max player, there are an exponential number of allocations in general.
Instead of a unit vector, we treat a pure strategy as an $mn$-dimensional vector which corresponds to a matrix in $\cA \subseteq \{0,1\}^{n \times m}$.
The set of mixed strategies of the max player can be represented as the convex hull of such vectors, which is denoted by $\Delta(\cA) \subseteq [0,1]^{n \times m}$.
Throughout the following, we regard $a \in \cA$ and $p \in \Delta(\cA)$ as column vectors of size $mn$.
Let $G \in \mathbb{R}^{mn \times n}$ be a payoff matrix defined as $G_{ie,j} = \mu_{ie}\mathbf{1}[i=j]$ for each $i,j\in [n]$ and $e\in M$.
Then $P^* = \max_{p \in \Delta(\cA)} \min_{q \in \Delta(n)} p^\top G q$.

We can regard that Algorithm~\ref{alg:proposed-ucb} executes a repeated process in which at round $t$, (a) the min player first choose a mixed strategy $q^t \in \Delta(n)$, (b) the max player then chooses $a^t \in \cA$, and (c) the both players observe $v^{t}_{ie} \sim D_{ie}$ for each $(i,e)$ such that $a^{t}_{ie} = 1$.
The regret of the max player is $\mathrm{Reg} = \max_{p\in \Delta(\cA)} \sum_{t=1}^T \langle p,  G q^t \rangle - \sum_{t=1}^T \langle a^t, G q^t \rangle $,
and that of the min player is $\mathrm{Reg}' = \sum_{t=1}^T \langle G^\top a^t,  q^t \rangle - \min_{q \in \Delta(n)} \langle G^\top a^t, q \rangle$.
We observe that 
\begin{align}
    R^\mu_T 
    \le T P^* - \min_{q \in \Delta(n)} \sum_{t=1}^T \langle a^t , G q \rangle 
    \le \max_{p \in \Delta(\cA)} \sum_{t=1}^T \langle p , G q^t \rangle - \min_{q \in \Delta(n)} \sum_{t=1}^T \langle a^t, G q \rangle = \mathrm{Reg} + \mathrm{Reg}'.
\end{align}
The first inequality here can be shown as follows:
the minimum value of
$\min_{q \in \Delta(n)} \sum_{t=1}^T \langle a^t , G q \rangle$ is attained by some pure strategy (unit vector) given $a^t$ ($t\in [T]$),
which means that $\min_{q \in \Delta(n)} \sum_{t=1}^T \langle a^t , G q \rangle$ is equal to $\ALG_{\mu}$ defined in Section~\ref{sec-setup}.
Furthermore,
as discussed in Section~\ref{sec-setup},
the value of
$T P^* = 
T \max_{p \in \Delta(\cA)} \min_{q \in \Delta(n)} p^\top G q =
T \min_{q \in \Delta(n)} \max_{p \in \Delta(\cA)} p^\top G q$
is an upper bound on $\OPT_{\mu}$.
Hence,
$T P^* - \min_{q \in \Delta(n)} \sum_{t=1}^T \langle a^t, G q \rangle $ is an upper bound on the surrogate regret $R^{\mu}_T$
and is exactly the value we analyzed in Section~\ref{sec-regret-analysis} to bound $R^{\mu}_T$.
Thus, if both $\mathrm{Reg} = o(T)$ and $\mathrm{Reg}'=o(T)$ can be achieved, then we will have $R^\mu_T \leq o(T)$.

Our contribution to algorithm design can be interpreted as the construction of a dynamics toward minimizing the regrets $\mathrm{Reg}$ and $\mathrm{Reg}'$ for each player.  
The challenge in achieving this lies in the fact that each player cannot directly observe their true loss or reward vectors --- specifically, $G q^t$ for the max player and $G^\top a^t$ for the min player.  
Instead, the observed values $v^t_{ie} \sim D_{ie}$ are noisy, and moreover, only a subset of the elements corresponding to the allocation specified by $a^t$ are observed in each round.  
To overcome this difficulty, the proposed algorithm ingeniously combines multiplicative weight updates with UCB techniques to efficiently control the regrets.  
More specifically, the min player applies a multiplicative weight update
$q^{t}_i\propto (1-\varepsilon)^{u^{t-1}_i/m}$ ($i\in [n]$)
based on UCB estimates $u^t$,
while the max player employs a combination of UCB and best-response dynamics with respect to the min player's action, aiming to achieve low regret for both players.

\end{document}